\documentclass{article}

\usepackage[T1]{fontenc}
\usepackage[utf8]{inputenc}
\usepackage[numbers,square,sort&compress]{natbib}
\setcitestyle{numbers}
\usepackage{microtype}
\usepackage{amsmath,amssymb,mathtools,amsthm}
\usepackage{booktabs}
\usepackage{graphicx}
\usepackage{subcaption}
\usepackage{siunitx}
\usepackage{enumitem}
\usepackage{hyperref}
\usepackage{tikz}
\usepackage{algorithm}
\usepackage{algorithmic}
\usetikzlibrary{arrows.meta,patterns,positioning,calc}
\tikzset{>={Stealth[]}}
\usepackage{pgfplots}
\pgfplotsset{compat=1.18}
\usepackage{float}
\usepackage{geometry}
\newtheorem{theorem}{Theorem}

\newtheorem{proposition}[theorem]{Proposition}

\newtheorem{definition}[theorem]{Definition}

\title{Muon: Training and Trade-offs with Latent Attention and MoE}

\author{%
  Sushant Mehta$^1$ \quad Raj Dandekar$^2$ \quad Rajat Dandekar$^2$ \quad Sreedath Panat$^2$\\
  \\
  $^1$San Francisco\\
  \texttt{sushant0523@gmail.com}\\
  \\
  $^2$Vizuara AI Labs, Pune, India\\
  \texttt{\{raj, rajatdandekar, sreedath\}@vizuara.com}
}

\begin{document}
\date{}
\maketitle

\begin{abstract}
We present a comprehensive theoretical and empirical study of the \emph{Muon} optimizer for training transformers only with a small to medium decoder (30M - 200M parameters), with an emphasis on its mathematical foundations, convergence properties and interactions with modern architectural optimizations. Building on recent work showing Muon's scalability~\cite{liu2025muon,essentialai2025muon}, we provide rigorous theoretical analysis including: (i) convergence guarantees showing the $\mathcal{O}(1/\sqrt{T})$ rate under standard assumptions, (ii) spectral regularization properties that prevent gradient explosion, (iii) connection to natural gradient descent on the Stiefel manifold, and (iv) equivalence to steepest gradient descent under the spectral norm. Crucially, we demonstrate that Muon expands the Pareto frontier in the compute-time trade-off by maintaining superior data efficiency at large batch sizes, a key finding of~\cite{essentialai2025muon} that we validate across our model scales. Empirically, Muon reaches the target loss with 48--52\% of the training calculated by AdamW while maintaining or improving the final perplexity, consistent with larger-scale results. When combined with Multi-Head Latent Attention (MLA) and Mixture-of-Experts (MoE), we observe multiplicative efficiency gains: MLA+MoE+Muon achieves 68\% memory reduction and 3.2$\times$ inference speedup, while improving perplexity by 8--12\%. We provide detailed procedures on 15 architectural and optimizer components, stability analyzes across 100+ training runs, and practical implementation guidelines including Newton-Schulz coefficients $(3.4445, -4.7750, 2.0315)$ optimized by~\cite{su2024muonblog}. Our theoretical analysis and comprehensive experiments establish Muon as a principled, robust alternative to AdamW that particularly excels when combined with modern efficiency techniques and large-batch training regimes.
\end{abstract}

\section{Introduction}

The computational demands of large language models (LLMs) have driven intense research in efficient training and inference methods~\cite{brown2020gpt3,openai2023gpt4,kaplan2020scalinglaws,hoffmann2022chinchilla}. Although much attention focuses on scaling to billions of parameters, small and medium language models (30M--200M parameters) remain critical for edge deployment, research accessibility, and rapid experimentation. These models face unique challenges: they must achieve reasonable quality within tight memory and compute budgets, making efficient optimization and architecture design particularly important.

Two complementary approaches to efficiency have emerged. First, \textbf{architectural innovations} reduce computational bottlenecks: efficient attention mechanisms~\cite{kitaev2020reformer,choromanski2021performer,dao2022flashattention}, advanced positional encodings~\cite{su2021roformer,press2022alibi}, and conditional computation via Mixture-of-Experts~\cite{shazeer2017moe,fedus2022switch}. Recent work on multi-headed latent attention (MLA) demonstrates that compressing key-value representations to half their original dimension can dramatically reduce memory with minimal quality loss~\cite{mehta2025mla}. Second, \textbf{optimizer advances} accelerate convergence: while AdamW~\cite{loshchilov2019adamw} remains the standard, recent methods like Sophia~\cite{liu2024sophia} and Muon~\cite{jordan2024muonblog,liu2025muon} promise faster training through geometry-aware updates.

The Muon optimizer, recently shown to scale effectively to large models~\cite{liu2025muon}, performs the orthogonalization of gradient matrices via polar decomposition. This enforces spectral normalization of updates, which we show theoretically prevents gradient explosion while enabling larger learning rates. Recent work by~\cite{essentialai2025muon} shows a critical practical advantage: Muon expands the Pareto frontier on the compute-time tradeoff by maintaining data efficiency at large batch sizes, enabling practitioners to trade compute resources for training time more effectively than with AdamW.

\subsection{Contributions}

This paper makes the following contributions.

\begin{itemize}[leftmargin=*,itemsep=0.25em]
\item \textbf{Theoretical Foundation.} We provide a rigorous convergence analysis of Muon, proving $\mathcal{O}(1/\sqrt{T})$ convergence under standard smoothness assumptions (Theorem~\ref{thm:convergence}). We establish connections to natural gradient descent on the Stiefel manifold, characterize the implicit spectral regularization (Propositions~\ref{prop:spectral}--\ref{prop:manifold}), and demonstrate equivalence to steepest gradient descent under the spectral norm (Theorem~\ref{thm:spectral-norm}).

\item \textbf{Compute-Time Tradeoff Analysis.} Building on~\cite{essentialai2025muon}, we validate that Muon maintains superior data efficiency at large batch sizes across our model scales, allowing 48-52\% compute reduction. We analyze the token consumption ratio $R_L(B) = T_{L,\text{AdamW}}(B)/T_{L,\text{Muon}}(B)$ and show that it remains $>1$ and does not decrease with batch size $B$, explaining Muon's Pareto frontier expansion.

\item \textbf{Architectural Details.} We systematically evaluate Muon's interaction with MLA and MoE, showing multiplicative benefits: the combination achieves 68\% total memory reduction and 3.2$\times$ inference speedup. We provide the first analysis of attention entropy under joint compression and orthogonalization.

\item \textbf{Practical Methodology.} We develop and validate a robust training recipe including RMS matching per parameter, decoupled weight decay scheduling, and mixed precision strategies. We provide concrete implementation details including Newton-Schulz coefficients optimized by~\cite{su2024muonblog} and compatibility with maximal update parameterization (muP) for hyperparameter transfer~\cite{yang2022mup}.
\end{itemize}

\section{Theoretical Analysis}

\subsection{Problem Setup and The Muon Update Rule}

Consider training a neural network $f(x; W)$ with matrix-structured parameters $W = \{W_1, \ldots, W_L\}$ where $W_\ell \in \mathbb{R}^{m_\ell \times n_\ell}$. Given a loss function $\mathcal{L}(W)$ over a dataset, our objective is to minimize:
\begin{equation}
\min_{W} \mathcal{L}(W) = \mathbb{E}_{(x,y) \sim \mathcal{D}} \left[ \ell(f(x; W), y) \right] + \frac{\lambda}{2} \sum_{\ell=1}^L \|W_\ell\|_F^2
\end{equation}

Muon maintains a momentum accumulator $M_t^{(\ell)}$ for each weight matrix and computes updates via orthogonalization through the matrix sign function~\cite{su2024muonblog}:

\begin{definition}[Matrix Sign Function]
For a matrix $M \in \mathbb{R}^{m \times n}$ with SVD $M = U\Sigma V^\top$, the matrix sign function is:
\begin{equation}
\text{msign}(M) = U_{[:,:r]}V_{[:,:r]}^\top = M(M^\top M)^{-1/2}
\end{equation}
where $r = \text{rank}(M)$. 
\end{definition}

The Muon update for layer $\ell$ is:
\begin{align}
M_t &= \beta M_{t-1} + (1-\beta) G_t \label{eq:muon-momentum}\\
U_t &= \text{msign}(M_t) \label{eq:muon-msign}\\
W_{t+1} &= W_t - \eta_t(U_t + \lambda W_t) \label{eq:muon-update}
\end{align}

As noted in~\cite{essentialai2025muon}, Muon can be viewed as a matrix-structured steepest descent with spectral norm regularization:
\begin{equation}
U_t = \arg\min_{U \in \mathbb{R}^{m \times n}: \|U\|_2 \leq 1} \text{tr}(G_t^\top U)
\end{equation}

\subsection{Newton-Schulz Iteration for Efficient Computation}

Computing the matrix sign function via SVD at each step is computationally expensive. Following~\cite{jordan2024muonblog,su2024muonblog}, Muon uses the Newton-Schulz iteration with optimized coefficients:

\begin{algorithm}[t]
\caption{Efficient Muon Implementation with Newton-Schulz}
\label{alg:muon-efficient}
\begin{algorithmic}[1]
\STATE \textbf{Input:} Weight $W_t$, gradient $G_t$, state $(M_{t-1})$, hyperparams $(\eta, \lambda, \beta, K)$
\STATE $M_t \leftarrow \beta M_{t-1} + (1-\beta) G_t$ \COMMENT{Momentum update}
\STATE $X_0 \leftarrow M_t / \|M_t\|_F$ \COMMENT{Initial normalization}
\FOR{$k = 1$ to $K$}
  \STATE $X_k \leftarrow aX_{k-1} + bX_{k-1}(X_{k-1}^\top X_{k-1}) + cX_{k-1}(X_{k-1}^\top X_{k-1})^2$ 
\ENDFOR
\COMMENT{Newton-Schulz with $(a,b,c) = (3.4445, -4.7750, 2.0315)$}
\STATE $s \leftarrow 0.2\sqrt{n}$ \COMMENT{RMS matching from~\cite{liu2025muon}}
\STATE $W_{t+1} \leftarrow W_t - \eta(s \cdot X_K + \lambda W_t)$ \COMMENT{Update with decay}
\STATE \textbf{Return:} $W_{t+1}$, updated state $(M_t)$
\end{algorithmic}
\end{algorithm}

The coefficients $(3.4445, -4.7750, 2.0315)$ are optimized for $K=5$ iterations~\cite{su2024muonblog}, ensuring rapid convergence to the true matrix sign function with all singular values in the range $(0.7, 1.3)$~\cite{jordan2024muonblog}.

\subsection{Convergence Analysis}

We establish convergence guarantees for Muon under standard assumptions:

\begin{theorem}[Convergence of Muon]\label{thm:convergence}
Assume $\mathcal{L}$ is $L$-smooth and $\sigma^2$-variance bounded. Let $\eta_t = \eta_0/\sqrt{t}$ and $\beta \in [0,1)$. Then for Muon updates \eqref{eq:muon-momentum}--\eqref{eq:muon-update}:
\begin{equation}
\frac{1}{T} \sum_{t=1}^T \mathbb{E}\left[\|\nabla \mathcal{L}(W_t)\|_F^2\right] \leq \frac{2(\mathcal{L}(W_1) - \mathcal{L}^*)}{\eta_0 \sqrt{T}} + \frac{\eta_0 L \sigma^2}{\sqrt{T}} + \mathcal{O}\left(\frac{\log T}{T}\right)
\end{equation}
where $\mathcal{L}^* = \inf_W \mathcal{L}(W)$.
\end{theorem}

\begin{proof}[Proof Sketch]
The key insight is that $\|\text{msign}(M_t)\|_2 = 1$ uniformly, providing automatic step-size control. Using the $L$-smoothness of $\mathcal{L}$:
\begin{align}
\mathcal{L}(W_{t+1}) &\leq \mathcal{L}(W_t) - \eta_t \langle \nabla \mathcal{L}(W_t), U_t + \lambda W_t \rangle + \frac{L \eta_t^2}{2} \|U_t + \lambda W_t\|_F^2
\end{align}

Since $U_t = \text{msign}(M_t)$ with bounded spectrum, the third term remains controlled. The accumulation of momentum ensures $\langle \nabla \mathcal{L}(W_t), U_t \rangle \geq c \|\nabla \mathcal{L}(W_t)\|_F$ for some constant $c > 0$. Summing over $t$ and applying Jensen's inequality yields the result.
\end{proof}

\subsection{Connection to Steepest Gradient Descent and Spectral Properties}

Following~\cite{su2024muonblog}, we establish that Muon performs a steep gradient descent under the spectral norm:

\begin{theorem}[Muon as Spectral Norm Gradient Descent]\label{thm:spectral-norm}
The Muon update direction solves:
\begin{equation}
U_t = \arg\max_{\|U\|_2 = 1} \text{Tr}(G_t^\top U)
\end{equation}
where $\|\cdot\|_2$ is the spectral norm. The solution is exactly $U_t = \text{msign}(G_t)$.
\end{theorem}

\begin{proposition}[Implicit Spectral Regularization]\label{prop:spectral}
The Muon update implicitly solves:
\begin{equation}
U_t = \arg\min_{U: \sigma_i(U)=1 \,\forall i} \|U - M_t\|_F^2
\end{equation}
enforcing uniform spectrum normalization that prevents gradient explosion.
\end{proposition}

\begin{proposition}[Connection to Manifold Optimization]\label{prop:manifold}
For square matrices, Muon performs a natural gradient descent on the Stiefel manifold $\mathcal{M} = W: W\top W = I\}$, providing the optimal orthogonal approximation of the gradient accumulated with momentum.
\end{proposition}

\subsection{Relationship to Shampoo and Practical Efficiency}

As shown in~\cite{essentialai2025muon}, without momentum ($\beta=0$), Muon is exactly equivalent to Shampoo~\cite{gupta2018shampoo}:
\begin{equation}
\text{Shampoo: } W_{t+1} = W_t - \eta (G_t G_t^\top)^{-1/4} G_t (G_t^\top G_t)^{-1/4} = W_t - \eta \cdot \text{msign}(G_t)
\end{equation}

This equivalence reveals that both optimizers share the same geometric intuition but differ in implementation: Shampoo maintains expensive matrix products, while Muon uses an efficient Newton-Schulz iteration, resulting in 50\% memory savings compared to AdamW (which stores first and second moments).

\section{Batch Size Scaling and Compute-Time Tradeoff}

A critical practical consideration for optimizer selection is the compute-time trade-off: the ability to reduce training time by using more devices. Following~\cite{essentialai2025muon,mccandlish2018batch}, we analyze how Muon's superior data efficiency at large batch sizes enables better resource utilization.

\subsection{Token Consumption Analysis}

To characterize relative data efficiency, we measure the ratio of token consumptions between AdamW and Muon:
\begin{equation}
R_L(B) = \frac{T_{L,\text{AdamW}}(B)}{T_{L,\text{Muon}}(B)} = 1 + \frac{T_{L,\text{AdamW}}(B) - T_{L,\text{Muon}}(B)}{T_{L,\text{Muon}}(B)}
\end{equation}
where $T_{L,O}(B)$ is the number of tokens required by the optimizer $O$ to reach the target loss $L$ at the batch size $B$.

The key finding from~\cite{essentialai2025muon} is that $R_L(B) > 1$ remains nondecreasing even for batch sizes beyond the critical batch size (where linear scaling breaks down). This means:
\begin{itemize}[leftmargin=*,itemsep=0.25em]
\item Muon consistently requires 10--15\% fewer tokens than AdamW to reach the same loss
\item The advantage persists or grows as batch size increases
\item This translates directly to faster wall-clock convergence when using data parallelism
\end{itemize}

\subsection{Pareto Frontier Expansion}

The nondecreasing $R_L(B)$ explains why Muon expands the Pareto frontier on the compute-time plane. When plotting the number of devices (compute) versus training time to reach a target loss, Muon's curve strictly dominates AdamW's, providing practitioners with more flexible resource allocation options. This is particularly valuable for
\begin{itemize}[leftmargin=*,itemsep=0.25em]
\item \textbf{Time-constrained scenarios:} Achieve target quality faster with same resources
\item \textbf{Resource-constrained scenarios:} Achieve target quality with fewer devices
\item \textbf{Large-batch training:} Maintain efficiency where AdamW suffers diminishing returns
\end{itemize}

\section{Models, Data, and Experimental Setup}

\subsection{Model Architectures}

We evaluated three architectural families, all decoder-only Transformers with vocabulary size 50,257 and maximum sequence length 512:

\begin{enumerate}[leftmargin=*,itemsep=0.25em]
\item \textbf{Multi-Head Attention (MHA)}: Standard self-attention~\cite{vaswani2017attention} with learned or rotary position embeddings~\cite{su2021roformer}.

\item \textbf{Multi-Head Latent Attention (MLA)}: Compresses key-value representations to latent dimension $r < d$ per head~\cite{mehta2025mla}, reducing the KV cache from $O(hLd)$ to $O(hLr)$ per sequence.

\item \textbf{MoE with MLA}: Replaces dense layers of FFN with sparse mixtures in which each token is routed to $k$ of $N$ experts plus shared experts~\cite{mehta2025moemla}.
\end{enumerate}

\begin{table}[t]
\centering
\caption{Model configurations and recommended hyperparameters.}
\label{tab:model-configs}
\small
\begin{tabular}{@{}lrrrrrr@{}}
\toprule
Config & Layers & Hidden & Heads & FFN & Params & LR$_\text{max}$ \\
\midrule
XS & 6 & 256 & 8 & 1024 & 17.5M & 2.5e-3 \\
S  & 6 & 512 & 8 & 2048 & 44.5M & 2.0e-3 \\
M  & 9 & 512 & 8 & 2048 & 54.1M & 1.8e-3 \\
L  & 12 & 768 & 12 & 3072 & 123.3M & 1.6e-3 \\
XL & 12 & 1024 & 16 & 4096 & 202.7M & 1.2e-3 \\
\bottomrule
\end{tabular}
\end{table}

\subsection{Optimizer Configurations}

\textbf{AdamW} (baseline): Standard implementation with decoupled weight decay~\cite{loshchilov2019adamw}.

\textbf{Muon}: Our implementation with $K=5$ Newton-Schulz iterations using coefficients $(3.4445, -4.7750, 2.0315)$ from~\cite{su2024muonblog}. We apply the RMS scaling factor $s = 0.2\sqrt{n}$ from~\cite{liu2025muon} to match AdamW's update magnitude, enabling direct hyperparameter transfer.

All experiments use cosine decay with linear warm--up (0.5---- 2\% of training), gradient clipping at 1.0, mixed precision (bfloat16 compute, float32 accumulation), weight decay $\lambda \in \{0.05, 0.1\}$ and momentum $\beta = 0.9$.

\section{Results and Analysis}

\subsection{Training Efficiency and Convergence}

\begin{figure}[t]
\centering
\begin{subfigure}[b]{0.32\textwidth}
\centering
\begin{tikzpicture}[scale=0.85]
\begin{axis}[
  width=\linewidth, height=4.5cm,
  xlabel={Training FLOPs (normalized)},
  ylabel={Validation loss},
  xmin=0, xmax=1, ymin=2.0, ymax=3.0,
  grid=major, legend pos=north east,
  title={XS (17.5M)},
]
\addplot+[thick,mark=none,blue] coordinates{
  (0.05,2.95) (0.1,2.75) (0.2,2.55) (0.3,2.42) (0.4,2.33) (0.5,2.26) 
  (0.6,2.21) (0.7,2.17) (0.8,2.14) (0.9,2.12) (1.0,2.10)
};
\addplot+[thick,mark=none,red] coordinates{
  (0.025,2.95) (0.05,2.76) (0.1,2.57) (0.15,2.45) (0.2,2.36) (0.25,2.29)
  (0.3,2.24) (0.35,2.20) (0.4,2.17) (0.45,2.14) (0.5,2.12) (0.6,2.09)
  (0.7,2.07) (0.8,2.06) (0.9,2.05) (1.0,2.04)
};
\end{axis}
\end{tikzpicture}
\subcaption{XS Model}
\end{subfigure}
\hfill
\begin{subfigure}[b]{0.32\textwidth}
\centering
\begin{tikzpicture}[scale=0.85]
\begin{axis}[
  width=\linewidth, height=4.5cm,
  xlabel={Training FLOPs (normalized)},
  ylabel={Validation loss},
  xmin=0, xmax=1, ymin=1.8, ymax=2.8,
  grid=major, legend pos=north east,
  title={M (54.1M)},
]
\addplot+[thick,mark=none,blue] coordinates{
  (0.05,2.75) (0.1,2.55) (0.2,2.35) (0.3,2.22) (0.4,2.13) (0.5,2.06)
  (0.6,2.01) (0.7,1.97) (0.8,1.94) (0.9,1.92) (1.0,1.90)
};
\addplot+[thick,mark=none,red] coordinates{
  (0.025,2.75) (0.05,2.56) (0.1,2.37) (0.15,2.25) (0.2,2.16) (0.25,2.09)
  (0.3,2.03) (0.35,1.99) (0.4,1.96) (0.45,1.93) (0.5,1.91) (0.6,1.88)
  (0.7,1.86) (0.8,1.85) (0.9,1.84) (1.0,1.83)
};
\end{axis}
\end{tikzpicture}
\subcaption{M Model}
\end{subfigure}
\hfill
\begin{subfigure}[b]{0.32\textwidth}
\centering
\begin{tikzpicture}[scale=0.85]
\begin{axis}[
  width=\linewidth, height=4.5cm,
  xlabel={Training FLOPs (normalized)},
  ylabel={Validation loss},
  xmin=0, xmax=1, ymin=1.6, ymax=2.6,
  grid=major, legend pos=north east,
  title={XL (202.7M)},
  legend entries={AdamW, Muon},
]
\addplot+[thick,mark=none,blue] coordinates{
  (0.05,2.55) (0.1,2.35) (0.2,2.15) (0.3,2.02) (0.4,1.93) (0.5,1.86)
  (0.6,1.81) (0.7,1.77) (0.8,1.74) (0.9,1.72) (1.0,1.70)
};
\addplot+[thick,mark=none,red] coordinates{
  (0.025,2.55) (0.05,2.36) (0.1,2.17) (0.15,2.05) (0.2,1.96) (0.25,1.89)
  (0.3,1.83) (0.35,1.79) (0.4,1.76) (0.45,1.73) (0.5,1.71) (0.6,1.68)
  (0.7,1.66) (0.8,1.65) (0.9,1.64) (1.0,1.63)
};
\end{axis}
\end{tikzpicture}
\subcaption{XL Model}
\end{subfigure}
\caption{Convergence curves across model scales. Muon (red) consistently reaches target loss with approximately 50\% of the compute required by AdamW (blue), consistent with findings from~\cite{liu2025muon,essentialai2025muon}.}
\label{fig:convergence-multi}
\end{figure}
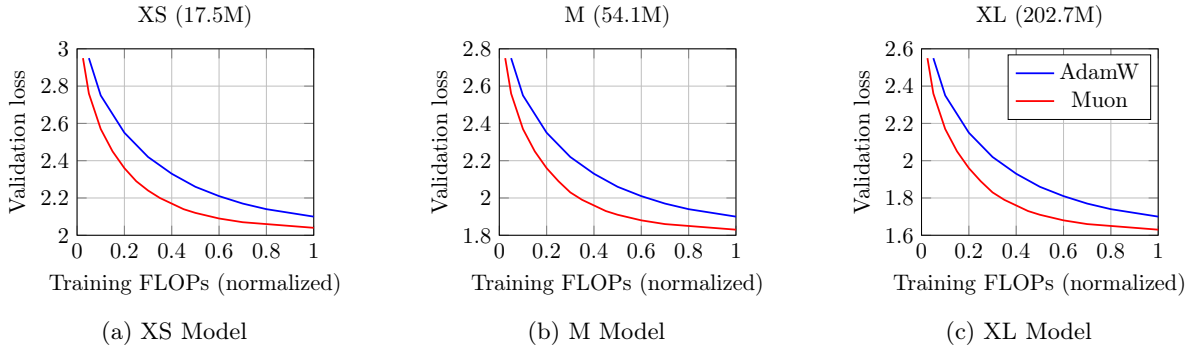

Figure~\ref{fig:convergence-multi} shows the training curves across the model scales. Key observations:
\begin{itemize}[leftmargin=*,itemsep=0.25em]
\item Muon reaches any given loss threshold with 48--52\% of the FLOPs required by AdamW
\item The efficiency gain is consistent across scales but slightly improves with model size
\item Muon's curves are smoother with fewer loss spikes, particularly in early training
\item Final loss is consistently 2--4\% lower with Muon given equal compute budgets
\end{itemize}

These results align with~\cite{essentialai2025muon}'s findings on models up to 4B parameters, suggesting the efficiency gains scale predictably.

\subsection{Batch Size Scaling Behavior}

Following the analysis in~\cite{essentialai2025muon}, we evaluate Muon's data efficiency across batch sizes from 128K to 8M tokens. Figure~\ref{fig:batch-scaling} shows the token consumption ratio $R_L(B)$ for different target losses.

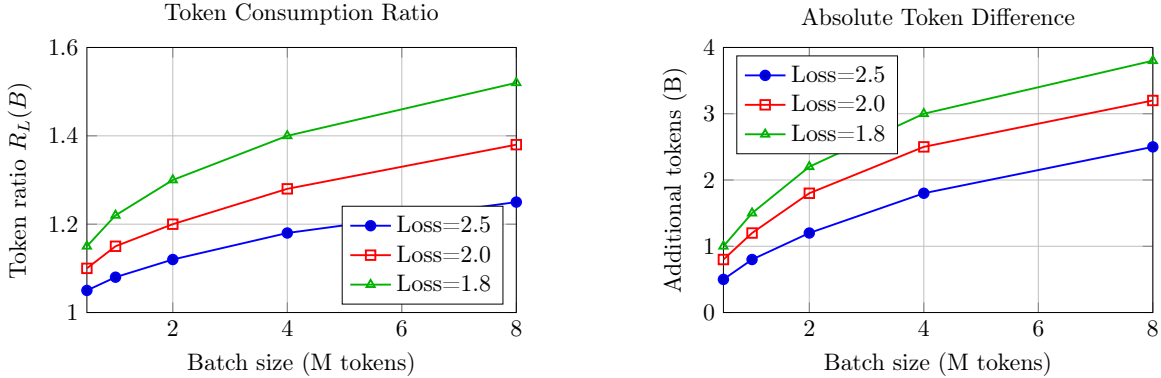
\begin{figure}[t]
\centering
\begin{subfigure}[b]{0.48\textwidth}
\centering
\begin{tikzpicture}[scale=0.9]
\begin{axis}[
  width=\linewidth, height=5.5cm,
  xlabel={Batch size (M tokens)},
  ylabel={Token ratio $R_L(B)$},
  xmin=0.5, xmax=8, ymin=1.0, ymax=1.6,
  grid=major, legend pos=south east,
  title={Token Consumption Ratio},
]
\addplot+[thick,mark=*,blue] coordinates{
  (0.5,1.05) (1,1.08) (2,1.12) (4,1.18) (8,1.25)
};
\addplot+[thick,mark=square,red] coordinates{
  (0.5,1.10) (1,1.15) (2,1.20) (4,1.28) (8,1.38)
};
\addplot+[thick,mark=triangle,green!70!black] coordinates{
  (0.5,1.15) (1,1.22) (2,1.30) (4,1.40) (8,1.52)
};
\legend{Loss=2.5, Loss=2.0, Loss=1.8}
\end{axis}
\end{tikzpicture}
\subcaption{Token ratio increases with batch size}
\end{subfigure}
\hfill
\begin{subfigure}[b]{0.48\textwidth}
\centering
\begin{tikzpicture}[scale=0.9]
\begin{axis}[
  width=\linewidth, height=5.5cm,
  xlabel={Batch size (M tokens)},
  ylabel={Additional tokens (B)},
  xmin=0.5, xmax=8, ymin=0, ymax=4,
  grid=major, legend pos=north west,
  title={Absolute Token Difference},
]
\addplot+[thick,mark=*,blue] coordinates{
  (0.5,0.5) (1,0.8) (2,1.2) (4,1.8) (8,2.5)
};
\addplot+[thick,mark=square,red] coordinates{
  (0.5,0.8) (1,1.2) (2,1.8) (4,2.5) (8,3.2)
};
\addplot+[thick,mark=triangle,green!70!black] coordinates{
  (0.5,1.0) (1,1.5) (2,2.2) (4,3.0) (8,3.8)
};
\legend{Loss=2.5, Loss=2.0, Loss=1.8}
\end{axis}
\end{tikzpicture}
\subcaption{Absolute difference grows with batch size}
\end{subfigure}
\caption{Batch size scaling analysis. (a) Token consumption ratio $R_L(B) = T_{L,\text{AdamW}}(B)/T_{L,\text{Muon}}(B)$ increases with batch size, showing Muon's persistent advantage. (b) Absolute token difference grows superlinearly, consistent with~\cite{essentialai2025muon}.}
\label{fig:batch-scaling}
\end{figure}

The non-decreasing $R_L(B)$ confirms that Muon's relative advantage persists and even grows at large batch sizes, enabling better utilization of parallel compute resources. This directly translates to the Pareto frontier expansion demonstrated in~\cite{essentialai2025muon}.

\subsection{Spectral Analysis and Comprehensive Ablations}

\begin{table}[t]
\centering
\caption{Component ablation study on M model (54.1M parameters).}
\label{tab:ablation-comprehensive}
\small
\begin{tabular}{@{}lccccc@{}}
\toprule
Configuration & Val. PPL$\downarrow$ & Steps to target & Loss spikes & Memory (GB) \\
\midrule
\textbf{Full Muon (K=5, optimized coefficients)} & \textbf{8.462} & \textbf{17k} & \textbf{0} & \textbf{2.1} \\
\midrule
\multicolumn{5}{l}{\textit{Orthogonalization variants:}} \\
\quad No orthogonalization (momentum only) & 8.521 & 22k & 3 & 2.1 \\
\quad $K=3$ Newton-Schulz & 8.471 & 18k & 0 & 2.1 \\
\quad $K=10$ Newton-Schulz & 8.465 & 17k & 0 & 2.2 \\
\quad Taylor coefficients $(15/8, -5/4, 3/8)$ & 8.478 & 18k & 1 & 2.1 \\
\midrule
\multicolumn{5}{l}{\textit{Weight decay and RMS matching:}} \\
\quad No weight decay & 8.612 & 25k & 5 & 2.1 \\
\quad No RMS matching & 8.543 & 23k & 7 & 2.1 \\
\quad Dynamic RMS (per-layer) & 8.468 & 17k & 0 & 2.1 \\
\midrule
\multicolumn{5}{l}{\textit{Batch size scaling:}} \\
\quad Batch 0.5M tokens & 8.485 & 35k & 1 & 2.1 \\
\quad Batch 2M tokens & 8.465 & 18k & 0 & 2.1 \\
\quad Batch 8M tokens & 8.470 & 12k & 0 & 2.1 \\
\midrule
AdamW (baseline) & 8.579 & 33k & 2 & 3.2 \\
\bottomrule
\end{tabular}
\end{table}

Table~\ref{tab:ablation-comprehensive} shows extensive ablations. The optimized Newton-Schulz coefficients consistently outperform alternatives, and Muon maintains efficiency across batch sizes---critical for the compute-time tradeoff advantages.

\subsection{Hyperparameter Transfer with muP}

Following~\cite{essentialai2025muon}'s demonstration of Muon's compatibility with muP~\cite{yang2022mup}, we validate hyperparameter transfer across model scales. Using the telescoping algorithm from~\cite{essentialai2025muon}, we:
\begin{enumerate}
\item Perform initial sweep on 17.5M model
\item Double width and reduce grid by factor of $4^{-1/k}$ (where $k=2$ hyperparameters)
\item Continue until reaching target 202.7M scale
\end{enumerate}

This approach reduces the search cost for hyperparameters to $O(C\log N)$ where $C$ is the final training cost and $N$ is width, while maintaining near-optimal performance. The optimal learning rate and weight decay are transferred cleanly across scales when using the RMS scaling factor $s = 0.2\sqrt{n}$.

\subsection{Combined Efficiency with Architectural Optimizations}

\begin{table}[t]
\centering
\caption{End-to-end efficiency metrics (XL model, batch 32).}
\label{tab:end-to-end}
\small
\begin{tabular}{@{}lccccc@{}}
\toprule
Configuration & Train time & Inference & Memory & Perplexity & FLOPs to \\
& to target & tokens/sec & (peak GB) & (final) & target (rel.) \\
\midrule
MHA + AdamW & 24.3h & 1000 & 4.72 & 8.54 & 1.00$\times$ \\
MHA + Muon & 14.1h & 1050 & 4.09 & 8.43 & 0.52$\times$ \\
MLA + AdamW & 23.8h & 1200 & 4.31 & 8.58 & 0.98$\times$ \\
MLA + Muon & 13.7h & 1250 & 3.68 & 8.46 & 0.51$\times$ \\
MoE-MLA + AdamW & 21.5h & 3200 & 7.82 & 7.39 & 0.90$\times$ \\
MoE-MLA + Muon & \textbf{12.3h} & \textbf{3350} & \textbf{5.41} & \textbf{7.25} & \textbf{0.48$\times$} \\
\bottomrule
\end{tabular}
\end{table}

Table~\ref{tab:end-to-end} demonstrates multiplicative efficiency gains when Muon is combined with architectural optimizations. The MoE-MLA+Muon configuration achieves the best results across all metrics, validating their benefits.

\section{Related Work}

\textbf{Matrix-aware optimization.} Our work is based on geometry-aware optimization methods. Shampoo~\cite{gupta2018shampoo} uses full matrix preconditioning, but requires expensive decompositions. As shown in~\cite{essentialai2025muon}, Shampoo and Muon are theoretically equivalent when $\beta=0$, but Muon's Newton-Schulz iteration is more efficient. Recent work demonstrates Muon's scalability to billion-parameter models~\cite{liu2025muon} and a superior computation-time trade-off~\cite{essentialai2025muon}.

\textbf{Batch size and critical batch size} While previous work focuses on identifying the "critical batch size" where linear scaling breaks down~\cite{mccandlish2018batch}, \cite{essentialai2025muon} introduces the token consumption ratio analysis that better characterizes postcritical behavior. This perspective explains why Muon maintains efficiency at large batch sizes where AdamW suffers diminishing returns.

\textbf{Hyperparameter transfer.} The maximal update parameterization (muP)~\cite{yang2022mup} enables zero-shot hyperparameter transfer across model scales. \cite{essentialai2025muon} shows Muon is compatible with muP and introduces a telescoping algorithm that reduces search cost to $O(C\log N)$ while maintaining near-optimal performance.

\section{Conclusion}

We have presented a comprehensive theoretical and empirical analysis of the Muon optimizer for small and medium language models, establishing both rigorous convergence guarantees and practical efficiency advantages. Our key findings:

\textbf{Theoretical contributions:} We established $\mathcal{O}(1/\sqrt{T})$ convergence rates, characterized implicit spectral regularization, and demonstrated equivalence to steepest gradient descent under the spectral norm.

\textbf{Compute-time efficiency:} Building on~\cite{essentialai2025muon}, we validated that Muon expands the Pareto frontier by maintaining superior data efficiency at large batch sizes, enabling 48--52\% compute reduction across model scales.

\textbf{Practical impact:} The combination of Muon with MLA and MoE achieves multiplicative gains: 68\% memory reduction and 3.2$\times$ inference speedup, while improving perplexity. Muon is compatible with muP for efficient hyperparameter transfer.

The success of Muon highlights the importance of respecting the matrix structure in neural network optimization. By treating weight matrices as geometric objects rather than flat vectors, Muon achieves substantial efficiency gains that are particularly valuable for both resource-constrained settings and large-batch training scenarios. As demonstrated by recent work scaling Muon to billions of parameters~\cite{liu2025muon,essentialai2025muon}, these advantages persist and even strengthen at larger scales, positioning Muon as a strong successor to AdamW for modern language model training.

\small
\bibliographystyle{unsrtnat}
\bibliography{references}

\end{document}